\documentclass[letterpaper, 10 pt, conference]{IEEEtran}

\IEEEoverridecommandlockouts
\usepackage{geometry}
\usepackage{cite}
\usepackage{amsmath,amssymb,amsfonts}
\usepackage{amsmath}
\usepackage{graphicx}
\usepackage{textcomp}
\usepackage{multirow}
\usepackage{xcolor}
\usepackage{algorithm}  
\usepackage{booktabs}
\usepackage{algorithm,algpseudocode}

\algnewcommand{\algorithmicforeach}{\textbf{for each}}
\algdef{SE}[FOR]{ForEach}{EndForEach}[1]
{\algorithmicforeach\ #1\ \algorithmicdo}
{\algorithmicend\ \algorithmicforeach}
\usepackage{amsmath}
\usepackage{threeparttable}
\usepackage{subfigure}
\usepackage{url}

\newtheorem{theorem}{\bf Theorem}[section]

\newenvironment{proof}{\quad{\noindent\it Proof.}\quad}{\hfill $\square$\par}

\usepackage{algorithm,algpseudocode}
\makeatletter
\newcommand{\algmargin}{\the\ALG@thistlm}
\makeatother
\newlength{\whilewidth}
\algdef{SE}[parWHILE]{parWhile}{EndparWhile}[1]
{\parbox[t]{\dimexpr\linewidth-\algmargin}{%
		\hangindent\whilewidth\strut\algorithmicwhile\ #1\ \algorithmicdo\strut}}{\algorithmicend\ \algorithmicwhile}%
\algnewcommand{\parState}[1]{\State%
	\parbox[t]{\dimexpr\linewidth-\algmargin}{\strut #1\strut}}

\def\BibTeX{{\rm B\kern-.05em{\sc i\kern-.025em b}\kern-.08em
    T\kern-.1667em\lower.7ex\hbox{E}\kern-.125emX}}

\title{\LARGE \bf
	Efficient reinforcement learning control for continuum robots\\ based on Inexplicit Prior Knowledge\\
	\thanks{}
}

\author{\IEEEauthorblockN{Junjia Liu$ ^\dagger $$^1$, Jiaying Shou$ ^\dagger $$^1$\thanks{$ ^\dagger $The first two authors Junjia Liu and Jiaying Shou contributed equally to this paper. (Corresponding author: Zhuang Fu, e-mail: zhfu@sjtu.edu.cn)},
		Zhuang Fu$^1 $, Hangfei Zhou$^1$, Rongli Xie$^2$, Jun Zhang$^2$ , Jian Fei$^2$ and Yanna Zhao$^3$}
	\IEEEauthorblockA{$1$ \textit{State Key Laboratory of Mechanical System and Vibration}, \textit{Shanghai Jiao Tong University}\\
		$2$ \textit{Ruijin Hospital affiliated to Shanghai Jiao Tong University School of Medicine}\\
		$3$ \textit{ Shanghai Ruijin Rehabilitation Hospital}\\
		Shanghai, 200240, China \\
		junjialiu@sjtu.edu.cn, Tabris@sjtu.edu.cn
	}
}

\begin{document}

\maketitle
\thispagestyle{empty}
\pagestyle{empty}

\begin{abstract}
	Compared to rigid robots that are generally studied in reinforcement learning, the physical characteristics of some sophisticated robots such as soft or continuum robots are higher complicated. Moreover, recent reinforcement learning methods are data-inefficient and can not be directly deployed to the robot without simulation. In this paper, we propose an efficient reinforcement learning method based on inexplicit prior knowledge in response to such problems. We first corroborate the method by simulation and employed directly in the real world. By using our method, we can achieve active visual tracking and distance maintenance of a tendon-driven robot which will be critical in minimally invasive procedures. Codes are available at \textit{https://github.com/Skylark0924/TendonTrack}.
	
\end{abstract}


\section{Introduction}
For decades, researchers have made massive efforts to make machines intelligent, in expectation of relieving human labors from repetitive, dangerous, and heavy work.
In traditional robotics, control of robots is realized by establishing kinematic and dynamic models in the form of a transformation matrix. This method has achieved excellent results in conventional robots with discrete rigid links but becomes difficult to implement when dealing with soft robots such as continuum robots. In the traditional method, several subjective assumptions have to be made to get control of continuum manipulators, leading to a deviation with actual circumstances and inaccurate in results\cite{Jones&Walker2006Kine}. Even though, the kinematic and dynamic models for continuum robots are often described in the form of nonlinear partial differential equations, which makes the control more complex.

Ever since reinforcement learning (RL) theory was proposed, developers have been trying to apply it to robotics. With introducing RL methods, it enhances the traditional method in rigid robotics with trial-and-error\cite{Qt-opt}\cite{Sim-to-real}. But the application of RL theory in continuum robots could still meet some resistance.
As far as we were concerned, recently only a few studies have applied RL to control continuum robots. In Thuruthe et al.’s research\cite{thuruthel2018model}, an accurate Vicon tracking system is provided for realizing closed-loop control from the third-person perspective. However, devices used in their research are not available for most application scenarios of continuum robots. Furthermore, data-inefficiency is the major drawback of RL algorithms, especially in a non-stationary continuum robot, which can make the learning on the real-world robot more impractical.

In this paper, we focalize automatic kinematics learning of complex robotic systems and end-to-end predicting control by using a visual servo from a first-person perspective. The primary problem we tackled is the data-efficiency of complex and non-stationary real-world robotics. We use the inexplicit prior knowledge to speed up the convergence of the learning process. Meanwhile, the ability of exploration is still guaranteed by an auto-adjusted exploitation coefficient.

To evaluate our proposed method empirically, we build a simulator by \textit{MuJoCo}\cite{mujoco} first and then try on a real-world continuum robot directly.
Our primary contributions are as follows:
\begin{itemize}
\item An efficient model-based RL framework for robotics that integrates inexplicit prior knowledge (IPK) is proposed. It guides the exploration followed the constrain of priors;
\item A Kalman filter based fusion controller fuses action distributions from priors and RL to achieve a safe exploration;
\item To balance the performance of priors and RL, we set an exploitation coefficient $\zeta$ that can adjust automatically according to the Kullback-Leibler (KL) divergence between two action distributions;
\item Results of simulation and experiment on real-world continuum robot demonstrate the data-efficiency of our method and require fewer interactions than the state-of-the-art model-based methods.  
\end{itemize}

\section{Related Works}
\subsection{Model-based Reinforcement Learning}
The word \textit{model-based} is easily ambiguous, which can both represent a given model in MPC and a learned model mainly used in RL. In this paper, model-based means a model learned from the trajectory data when either the system dynamic model or the environment model is unknown. 

Model-based reinforcement learning (MBRL) began with Dyna \cite{Dyna} architecture. Compared to model-free reinforcement learning (MFRL), it is undoubtedly more suitable for robotic systems because of the data-efficiency of taking full advantage of experience data. Since MBRL uses a learned dynamic model to promote the learning process, its uncertainty will bring incorrect transition and impair value function approximation\cite{kalweit2017uncertainty}. MVE\cite{MVE} controls the uncertainty of the model by limiting the imagination of the model to a fixed depth. STEVE\cite{STEVE} improves the thought of MVE by dynamically interpolating between model rollouts of a different horizon length of each example and ensures that models are used without redundant errors.

Furthermore, probabilistic models that usually relay on Bayesian methods are more suitable for robotics issue since they can combine the uncertainty into model building\cite{chatzilygeroudis2019survey}. Black-DROPS\cite{chatzilygeroudis2017black} and PILCO\cite{PILCO} both utilize Gaussian Processes (GPs) to reduce the interaction time and solve several robotics tasks. On the basis of GPs, Bayesian
NN (BNN)\cite{gal2016dropout} has been used in some work to improve the scaling of MBRL algorithms\cite{gal2016improving}. Chua et al. \cite{chua2018deep} propose a combination of ensembles and BNN recently for learning probabilistic dynamics models of higher dimensional systems.

Most recently, Michael et al.\cite{janner2019trust} propose a monotonic model-based policy optimization (MBPO) algorithm. MBPO combines the benefits of adaptive length planning and ensemble BNN models to provide a performance guarantee and get a state-of-the-art efficient performance on several common RL tasks.

\subsection{Reinforcement Learning with prior knowledge}
Although MBRL algorithms achieve infusive success, they still take too many time steps (e.g., the state-of-the-art MBRL method MBPO still needs 5k steps even for a simple Pendulum task) which still impractical in real-world robot application. 
Recently, a comprehensive survey on policy search algorithms for learning robot controllers in a handful of trials is worth reading \cite{chatzilygeroudis2019survey}. Except for creating data-driven surrogate models as MBRL algorithms do, the article states that there is another way to let robots adapt with \textit{micro-data}: leverage prior knowledge on the policy parameters\cite{billard2008survey}\cite{osa2018algorithmic}, on the expected return\cite{cully2015robots}, or on the dynamic models\cite{chatzilygeroudis2018using}\cite{cutler2015efficient}. We can bring some prior knowledge of the robot system in for both stable and efficient, rather than merely learning from scratch.

Similarly, as flourishing fields, Imitation Learning (IL) and Reinforcement Learning from Demonstrations (RLfD) also use expert demonstrations as a prior for accelerating the training process. They integrate expert data by behaviour cloning\cite{schaal1997learning}\cite{abbeel2004apprenticeship}\cite{jing2019task}, data augmenting\cite{DAgger}, or setting as a policy penalty\cite{brys2015reinforcement}\cite{kang2018policy} or a constraint\cite{mareinforcement}.

In addition, priors can use in a stronger way for some tasks. Moreno et al.\cite{moreno2004prior} add a set of prior knowledge sources as a basic controller and use a credit assignment block to judge when to explore by RL. However, the evaluation function is designed by hand and just acts as a conditional judgment.

\subsection{Continuum Robot Control}
Continuum robots have many usages in flexible scenes, especially in interventional medicine field, because they theoretically have infinite Degrees of Freedom (DOF)\cite{burgner2015continuum}.

Studies on control of continuum robots have been widely explored in traditional methods\cite{Survey2018}. Researchers tend to establish the manipulator kinematic and dynamic models derived from several geometric assumptions. The most commonly used model simplifies the control issues based on the constant curvature (CC) approximation and linearized feedback\cite{Jones&Walker2006Kine}\cite{hannan2003kinematics}. This CC model performs worse when external loads are non-negligible\cite{Gravagne2003deflection}\cite{trivedi2007geometrically}. As an alternative, mechanics-modified models were used in continuum robotics. Walker, Hannan and Gravagne have introduced the hyper-redundant robotics\cite{Gravagne2000}\cite{hannan2000novel} and large-deflection dynamic model was used in their researches\cite{Gravagne2003deflection}. Considering the backbone of continuum robots as an elastic rod, Webster et al.\cite{rucker2011statics} and Mahvash et al.\cite{mahvash2011stiffness} have respectively applied Cosserat rod model in their researches. Although an increase in accuracy is found, solutions of those models, described in the form of nonlinear partial differential equations, are sensitive to parameters and time-consuming\cite{trivedi2007geometrically}\cite{chikhaoui2019compare_model}, which inevitably increases the complexity of the control issues in continuum robotics.

\section{Problem Formulation}
Like other RL problems, we need first simplify the sequential decision procedure as a Markov Decision Process (MDP). Then abstract it into $<\boldsymbol o,\boldsymbol a,\boldsymbol o',r>$:
\begin{itemize}
	\item \textbf{Observation} $\boldsymbol o$: Observation contains two parts: \textbf{states of the target} and \textbf{accumulated actions}. States of target are obtained from the view scope of the pinhole camera at the end of the continuum. They consist of the horizontal and vertical coordinate components $w, l$  and the Euclidean distance $h$ between the camera and the target, as shown in Fig.\ref{scope}. Accumulated actions are obtained by adding executed actions throughout the trajectory;
	\item \textbf{Action} $\boldsymbol a$: Actions are the control values of four motors, each is between (0, 1);
	\item \textbf{Reward} $r$: The raw reward is the amount of observation change between current observation and next observation. There are also some special scenes need to amend the raw reward which are listed below:
\end{itemize}
$$ R=\left\{
\begin{array}{rcl}
r-10       &      & {\text{any}(s') \text{ is None}}\\
r-10       &      & \text{out of view scope}\\
r+10       &      & {|s_0'| < \epsilon\text{ and }|s_1'|<\epsilon}\\
r+100      &      & {\text{finish the tracking route}}
\end{array} \right. $$
\begin{figure}[h]
	\centerline{\includegraphics[width=0.6\linewidth]{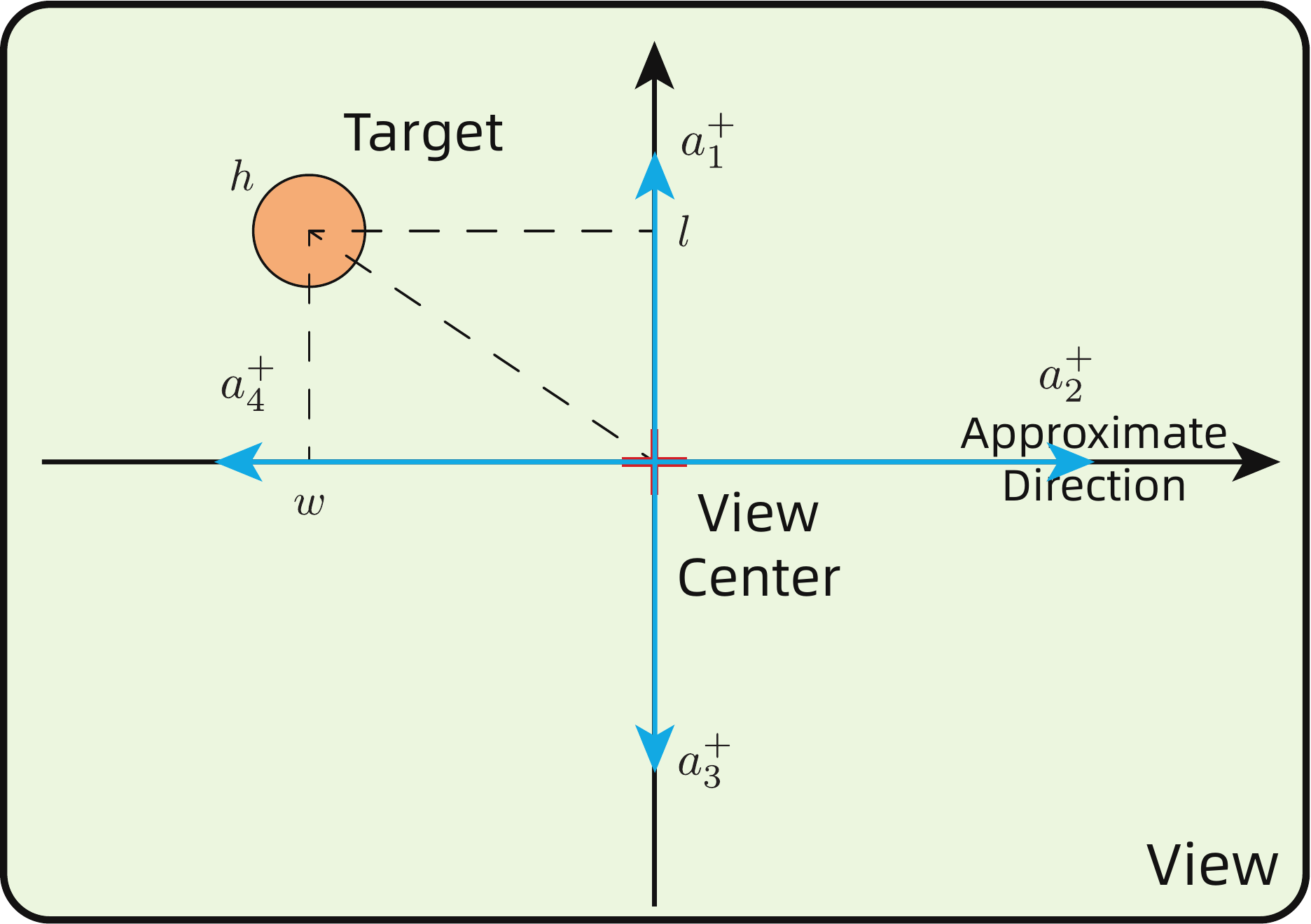}}
	\caption{View scope from the pinhole camera at the end of the continuum robot. A coordinate system is established with approximate directions of four motors of the continuum robot as the coordinate axis (see the details in IV.A). $a_n^+$ represents the positive direction of motor $n$ in our task. Distance $h$ is obtained by a simple digital image process technique.}
	\label{scope}
\end{figure}

As we want to introduce both model-based and prior knowledge into robotics RL problems, the strategy of building surrogate models and leveraging prior knowledge (Section IV.A) need to be selected carefully. In this article, we build a probabilistic surrogate model for a dynamic system. We model the robots as discrete-time dynamical systems that can be described by transition probabilistic of the form:
$$\boldsymbol o_{t+1} \sim p(\boldsymbol{o}_{t+1}|\boldsymbol o_t,\boldsymbol a_t)$$

Use the trajectories observed so far to learn ensemble BNN models. Given $n$ observation-action pairs $\tilde{\mathbf{X}}=\left[
\tilde{\mathbf{x}}_{1}, \ldots, \tilde{\mathbf{x}}_{n}
\right],  \tilde{\mathbf{x}}_t = (\boldsymbol o_t,\boldsymbol a_t)$ as training inputs and corresponding training targets $\mathbf{y}=\left[\Delta_{1}, \ldots, \Delta_{n}\right], \Delta_{t} = \boldsymbol o_{t} - \boldsymbol o_{t-1}$. The reason why we use differences as model targets is that differences are easier to learn than original observation. Learning differences is similar as learning gradients. Therefore, we can use these models to estimate the probabilistic mentioned above and make a planning\cite{depeweg2016learning}.

\section{RL based on Inexplicit Prior Knowledge}
Before elaborating on our method, imagine a simple daily scene. When you are asking for directions in an unknown area, you can just get an answer like “go straight and turn left at the third corner”, rather than an explicit map. In the beginning, you might be confused about the whole trajectory, but when you go straight, you will definitely know when to turn left!

Humans always have some intuitive inexplicit prior knowledge (IPK) about how to control robots, like “it (a rigid manipulator) should raise the third joint to catch the ball” or “it (a mobile robot) should go forward and turn right to enter the room”. It might be inaccurate, but is generally on the right path. To avoid useless exploration in a complex manipulator system, the general trend of movement can be pointed out and taught to the ignorant robot. All it needs to do is continuing amending the movement trend mapping from data and finally get reliable and explicit mapping. According to this idea, the main framework of our method is shown in Fig. \ref{main}. The full algorithm is outlined in Appendix C Algorithm \ref{Pseudocode}.
\begin{figure}[h]
	\centerline{\includegraphics[width=1\linewidth]{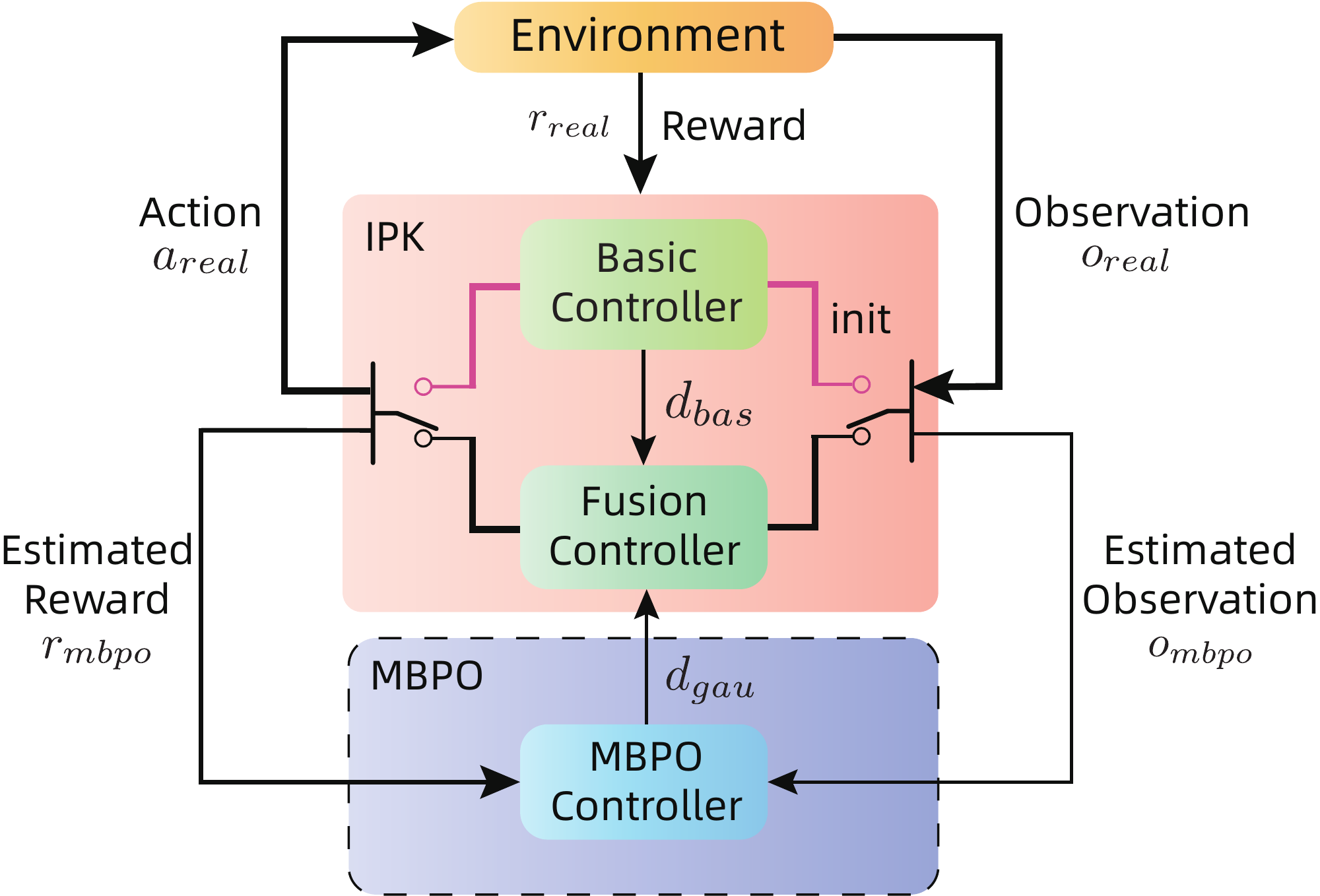}}
	\caption{The main framework of our method contains two parts: IPK subsystem and MBPO  subsystem. Light pink IPK subsystem, which includes both Basic Controller for initial exploration and Fusion Controller for stable long-term control, is the real agent policy to interact with the environment. The lilac MBPO subsystem obtains estimated rewards and observations from IPK, outputs the action distribution of its policy to fuse with basic action distribution, and finally constitutes a fusion policy.}
	\label{main}
\end{figure}

\subsection{Exploration guided by the prior knowledge}
\begin{figure}[h]
	\centerline{\includegraphics[width=0.8\linewidth]{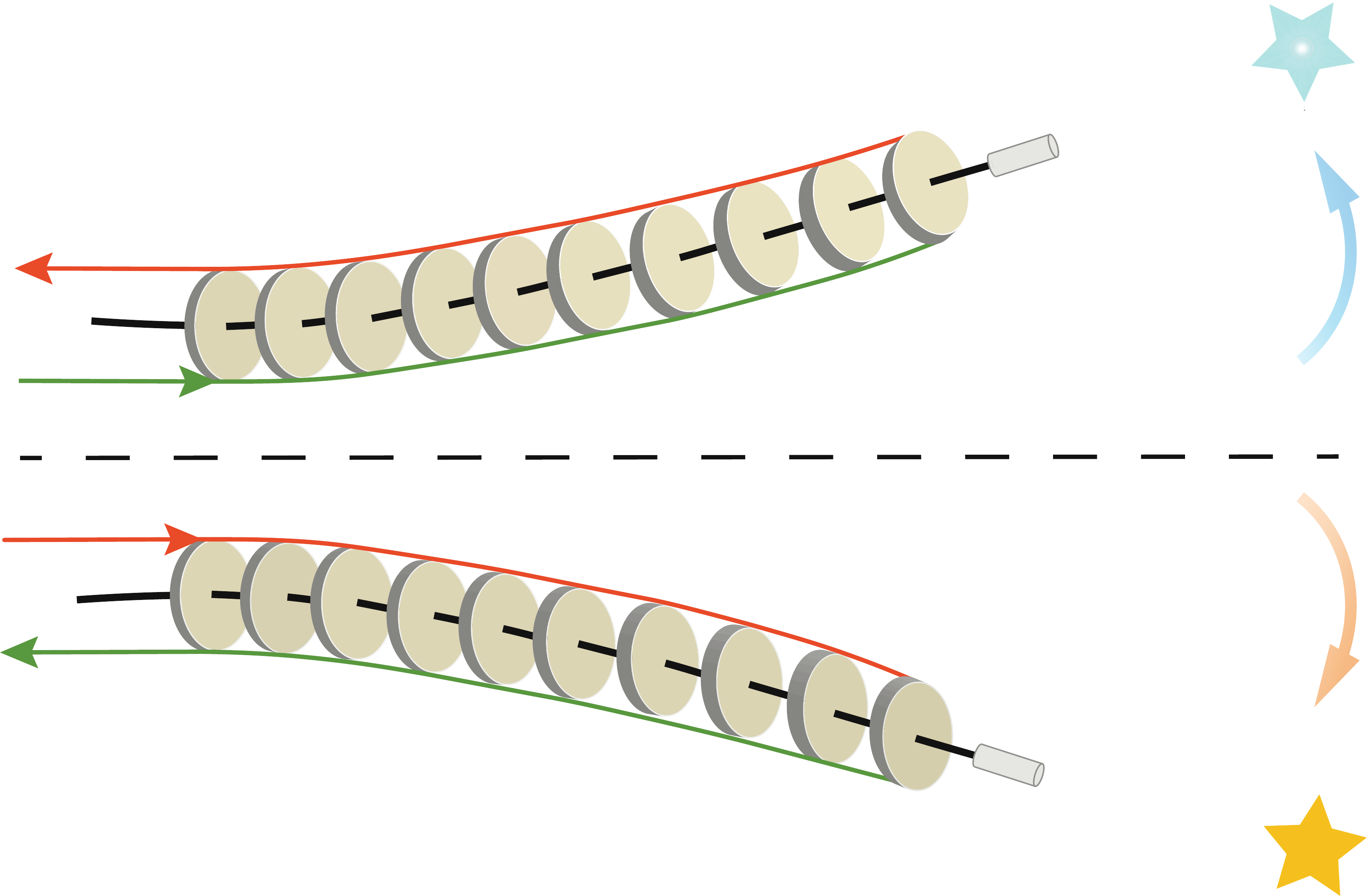}}
	\caption{Take the two-DOF tendon driven continuum robot as an example, the blue star in the top figure is above the central axis, thus the red tendon needs to be wound up, and the green one needs to be released. The bottom figure vice versa. At the end of the black elastic rod, there is a pinhole camera to provide a first-person perspective.}
	\label{continuum}
\end{figure}
\begin{figure*}[htbp]
	\centerline{\includegraphics[width=0.6\linewidth]{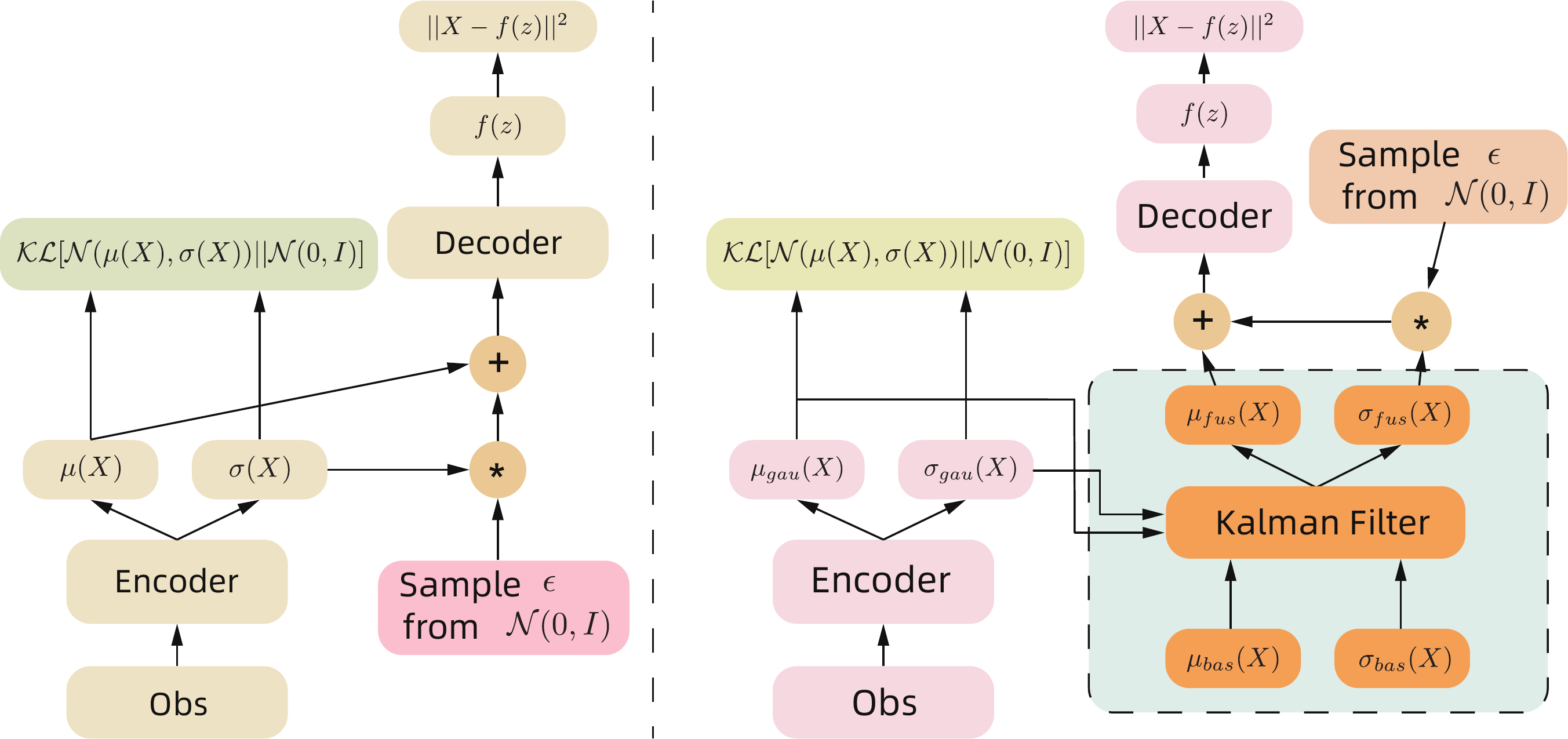}}
	\caption{Left: Reparameter trick in SAC paper; Right: Before the reparameter trick, the output Gaussian distribution from the MBPO controller is fusing with distribution from the basic controller by Kalman filter.}
	\label{Fusion_controller}
\end{figure*}

The so-called \textit{inexplicit} in this article represents the \textit{approximate} direction of each tendon motor. As shown in Fig. \ref{continuum}, the motor at the end of the red line determines the rise of the robot, and the green one determines the downwards direction. Thus, we can build a coordinate axis in Fig. \ref{scope} using these approximate directions. This kind of information is much easier to obtain than a kinematic model and certainly not precise. They can be obtained by powering up each motor and recording their specialty. Different from the expert demonstrations in IL and RLfD fields, the condition of inexplicit prior knowledge is stronger. This prior knowledge can be regarded as a \textbf{basic rule-based controller}. It provides the simplest way to control a robot. Before each interaction, the target direction is first confirmed from the first-person perspective camera. Then calculate horizontal and vertical coordinate components $(w, l)$. Finally, randomly select a motor in that direction to perform the motion $(a_4^+ \text{ \textit{or} } a_2^-, a_1^+  \text{ \textit{or} } a_3^-)$.

Because of suffering from the inaccuracy and not wasting time on tuning a PID function, this basic controller has many drawbacks:
\begin{itemize}
	\item When the target is close to the center of view, it is easy to oscillate around the target and never converge;
	\item As the degree of bending becomes greater, the approximate is much more inaccuracy, especially when bending from two vertical directions acts  simultaneously;
	\item Since motors in the same direction are selected randomly in the basic controller, it is powerless for achieving the distance keeping task.
\end{itemize}

Therefore, we also need RL exploration to amend the observation-action mapping. The integration of priors intends to prolong the length of the RL task horizon and tries to make it possible to sample more successional observation-action pairs. In this paper, we adopt the soft actor-critic (SAC) \cite{SAC} as our policy gradient algorithm and MBPO as our MBRL algorithm. The primary procedure of MBPO is to employ a uniform policy that generates random actions to guarantee the exploration scope. However, this will lead to a major risk of a robot crash and may cost a tremendous amount of time to reset. Both of them are insufferable in a real-world application. 

We tackle this by setting two sets of interaction outputs, one from the IPK basic controller and another from MBPO. As shown in Fig.\ref{main}, the replay buffer is augmented from $\langle o, a, r, o'\rangle$ to $\langle o_{real}, a_{mbpo}, r_{mbpo}, o_{mbpo}', a_{real}, r_{real}, o_{real}'\rangle$, where the subscript $real$ stands for information from IPK subsystem which does the real interaction. Both $real$ and $mbpo$ contain two parts, $real$ contains \textbf{initial exploration procedure} $bas$ and \textbf{fusion procedure} $fus$, e.g., $o_{real}=\{o_{bas}, o_{fus}\}$. Correspondingly,  the subscript $mbpo$ contains \textbf{uniform initial exploration procedure} $uni$ (which is not shown in Fig.\ref{main}) and \textbf{Gaussian policy learning procedure} $gau$, e.g., $o_{mbpo}=\{o_{uni}, o_{gau}\}$. Actions from IPK subsystem is used for practically interacting with the environment and get the real reward $r_{real}$ and the real next observation $o_{real}'$. In contrast, MBPO information is merely used in policy updates.

We first describe the exploration approach of the initial exploration procedure in this chapter and elaborate on the fusion and learning procedure later.
According to $\langle o, a_{bas}, r_{bas}, o_{bas}'\rangle$, the approximation of reward $r_{uni}$ and next observation $o_{uni}'$ can be estimated by local linearity:

\begin{equation}\label{r_uni}
\begin{aligned}
r_{uni} &= \sum^n_{i\in \mathcal{A}} (a_{uni}^i-a_{bas}^i)+ r_{real}\\
o'_{uni} &= \dfrac{o'_{bas}-o}{a_{bas}}\times a_{uni}
\end{aligned}
\end{equation}

Intuitively, IPK actions guarantee that robots will eventually reach their target with a high probability, and the MBPO part can still improve its policy with a certain degree of precision. Therefore, the initial exploration procedure achieves a safe exploration. It implements once and gain twice experience, it is obviously more efficient than the vanilla MBPO does (more details in IV.C Data Augment).

\subsection{Fusion Controller}
After the initial exploration procedure, MBPO trains a Gaussian policy as the main learning policy which outputs Gaussian distributions of actions. Correspondingly, the IPK subsystem also turns into a new link: Fusion Controller. 

Although MBPO disentangles the task horizon and model horizon by querying the model only for short rollouts, it is still limited by the probability of reaching the target, especially in sparse reward problem. Since the IPK basic controller is rule-based, it is convenient to assess its performance. From the initial replay buffer and the log of their task length, we can revert the data to the full form. At each time step, we can get the vector of target both before and after the action, then the deviation of each action from anticipative direction can be easily estimated. These deviations can be depicted as a Gaussian distribution (more details in IV.C Basic Controller Accuracy Estimation). Moreover, the raw actions of SAC are also Gaussian distribution. How can we use both of these useful information?

A very naive thought is fusing the basic output Gaussian with the SAC action distribution. Kalman filter is a common method to fuse the measurement information of multiple sensors and tend to be more accurate than each of them. As Fig. \ref{Fusion_controller} shown, we use a Kalman filter to integrate outputs from both two controllers and acquire a new fusion distribution. This procedure is before the reparameter trick of SAC.
\begin{equation}\label{kalman}
\begin{aligned}
\mu_{fus} &= \dfrac{\sigma^2_{bas}\times \mu_{gau} +\sigma^2_{gau}\times \mu_{bas}}{\sigma^2_{bas}+ \sigma^2_{gau}}\\
\sigma^2_{fus}&= \dfrac{1}{\dfrac{1}{\sigma^2_{bas}}+\dfrac{1}{\sigma^2_{gau}}}
\end{aligned}
\end{equation}

Our motivation for introducing the IPK subsystem is to demonstrate and guide the MBRL algorithm in order to reduce wasting time on useless exploration at the beginning, but not limit it. Because some motions, like axial distance maintenance and real-time tracking, cannot gain enough information from IPK basic controller, they still need relay on the exploration. So the MBPO reward estimation here is more complicated. We set an exploitation coefficient $\zeta$ to balance exploration and exploitation which is inspired by the temperature coefficient $\alpha$ in MBPO.
\\
\begin{theorem}[Exploitation Coefficient Auto-Adjustment]\label{ecaa}
	Let $\mathcal{G}_{gau}$ be the Gaussian distribution from the $T-1$ Gaussian policy and let $\mathcal{G}_{bas}$ be the basic action distribution with uncertainty. Then the exploitation coefficient is related to the KL-divergence between these two distributions.
\begin{equation}
\begin{aligned}
	\zeta_{bas}^{T^*}&=\arg\min_{\zeta^{T}\geq  0} \mathbb{E}_{s_{T-1},a_{T-1}\sim\rho_\pi^*}\\
	&\{-\zeta^{T}_{bas}\mathcal{D}_{KL}[\mathcal{G}_{bas}(\pi^{T-1}_{bas})||\mathcal{G}_{gau}(\pi^{T-1}_{gau})]-\zeta^{T}_{bas}\mathcal{D}_0\}\\
	\zeta_{real}^{T^*}&=1-\zeta_{bas}^{T^*}
\end{aligned}
\end{equation}
where $\mathcal{D}_0$ is a target divergence for KL-divergence limiting.\\
\end{theorem}
\begin{proof}
See Appendix A.1.

\end{proof}

By introducing KL-divergence into the reward function, the behaviour of the MBPO controller will similar to the basic controller gradually. Moreover, we set a negative target divergence $\mathcal{D}_0$ to overcome the uncertainty of basic controller and trade-off exploration and exploitation. This approach is analogous to the behaviour cloning method in IL field\cite{DAgger}.
\begin{equation}
\begin{aligned}
	r_{gau}= \zeta_{bas}[- \mathcal{D}_{KL}(\mathcal{G}_{bas}||\mathcal{G}_{gau})- \mathcal{D}_0+r_{real}]  + \zeta_{real}r_{real}
\end{aligned}
\end{equation}

Meanwhile, the exploitation coefficient also should be used in Equation \ref{kalman} as a weight parameter and structure the final fusion function:
\begin{equation}\label{weighted_kalman}
\begin{aligned}
\mu_{fus} &= \dfrac{\zeta_{real}\times \sigma^2_{bas}\times \mu_{gau} +\zeta_{bas}\times\sigma^2_{gau}\times \mu_{bas}}{ \zeta_{real} \times\sigma^2_{bas}+ \zeta_{bas} \times \sigma^2_{gau}}\\
\sigma^2_{fus}&= \dfrac{1}{\zeta_{bas}\times\dfrac{1}{\sigma^2_{bas}}+\zeta_{real}\times\dfrac{1}{\sigma^2_{gau}}}
\end{aligned}
\end{equation}

The policy evaluation step is similar to Soft Policy Evaluation\cite{SAC}, it ensures that we can obtain soft value function for any policy $\pi$. However, we need to prove that the new policy will achieve higher value than the old one by limiting the KL-divergence.
\\
\begin{theorem}[Fusional Policy Improvement]
	According to Theorem \ref{ecaa} and Equation \ref{weighted_kalman}, let $\pi_{T}\triangleq \pi(a_t|s_t,\zeta_{bas}^{T-1})$, the new policy of $T+1$ time step is $\pi_{T+1}\triangleq \pi(a_{T+1}|s_{T+1},\zeta_{bas}^{T})$. Then $Q^{\pi_{T+1}}\left(\mathbf{s}_{t}, \mathbf{a}_{t}\right) \geq Q^{\pi_{T}}\left(\mathbf{s}_{t}, \mathbf{a}_{t}\right)$ for all $\left(\mathbf{s}_{t}, \mathbf{a}_{t}\right) \in \mathcal{S} \times \mathcal{A}$ with $|\mathcal{A}|<\infty$.\\
\end{theorem}
\begin{proof}
	See Appendix A.2.

\end{proof}

Implement Soft Policy Evaluation and Fusion Policy Improvement repeatedly, the policy will eventually converge to the optimal as proved in SAC Theorem 1\cite{SAC}.

In conclusion, RL controller is exploring and exploiting beneath the IPK subsystem, interaction with the real-world is IPK subsystem's business. This design brings two advantages:\textbf{ safe interaction} and \textbf{exploration guided by human priors}.

\subsection{Implementation Details}
\textbf{KL-divergence between Multivariate continuous Gaussian distributions.} Since the action dimension is related to DOF of the continuum robot, the action distribution is multivariate. In this task, our continuum robot has 4-DOF and we simply merge the two of same direction action distributions by sum up their means and average their variances. This would relax the restriction of motor selection in the same direction and help to explore more unaware trajectories which are not included in IPK. The KL-divergence between the simplified bi-variate Gaussian distributions can be represented by their means and covariance matrices:
\begin{equation}\label{KL}
\begin{aligned}
\mathcal{D}_{KL}&(\mathcal{G}_{bas}||\mathcal{G}_{gau})\\&=\frac{1}{2}\left[ \log \frac{\left|\boldsymbol\Sigma_{2}\right|}{\left|\boldsymbol\Sigma_{1}\right|}-n+\operatorname{tr}\left(\boldsymbol\Sigma_{2}^{-1} \boldsymbol\Sigma_{1}\right)+\bar{\boldsymbol\mu}^{T} \boldsymbol\Sigma_{2}^{-1}\bar{\boldsymbol \mu}\right] 
\end{aligned}
\end{equation}
where $\bar{\boldsymbol\mu}=\boldsymbol\mu_{2}-\boldsymbol\mu_{1}$, the subscript $1$ for IPK basic policy and $2$ for MBPO Gaussian policy, $n$ represents the number of variables and here $n=2$.

\textbf{Data Augment.} The well-designed structure allows us to obtain two experience simultaneously in one interaction. The augmented dataset $\langle o_{real}, a_{mbpo}, r_{mbpo}, o_{mbpo}', a_{real}, r_{real}, o_{real}'\rangle$ can be divided into two-part, $\langle o_{real}, a_{mbpo}, r_{mbpo}, o_{mbpo}'\rangle$ and  $\langle o_{real}, a_{real}, r_{real}, o_{real}'\rangle$, and sample in proportion according to the trust of MBPO controller. The trust is reflected in the exploitation coefficient $\zeta$ we mentioned above. Data from the MBPO subsystem is gained by estimated, therefore it is necessary to assess its data availability by $\zeta$. After weighted sampling, two-part of data is concatenated together and used to train both model-free SAC network and model-based ensemble BNN models. This ensures the policy to balance the real interaction and MBPO exploration beneath the interaction.

\textbf{Basic Controller Accuracy Estimation.} After initial exploration procedure, we can estimate the accuracy of the basic controller by the trajectory data in replay buffer. Since $\boldsymbol{o}_t$ contains the direction of the target at time $t$, and $\boldsymbol{o}_{t+1}'$ has the direction of time $t+1$. We can use these information to estimate the performance of $\boldsymbol{a}_t$ by calculate the angle $\theta$ between two directions (Fig. \ref{scope2}). For example, if $\boldsymbol b_t=(w_t, l_t), \boldsymbol b_{t+1}=(w_{t+1}, l_{t+1}), \boldsymbol{a}_t = (0, a_{t,2}, a_{t,3}, 0)$, then the real action direction of time $t$ is $\boldsymbol d_t = \boldsymbol b_{t+1} - \boldsymbol b_{t}$:
$$\boldsymbol c_t  = \dfrac{-\boldsymbol b_t}{||\boldsymbol b_t||} \times ||\boldsymbol d_t|| \times \cos \theta =  \boldsymbol b_{t} \cdot \dfrac{\boldsymbol d_t\cdot \boldsymbol b_t}{||\boldsymbol b_t||^2}$$
$$\text{deviation } \boldsymbol e_t = \dfrac{\boldsymbol d_t - \boldsymbol c_t}{||\boldsymbol d_t||}$$
$${\text{deviation per unit action} \dfrac{e_{t,x}}{a_{t,2}} \dfrac{e_{t,y}}{a_{t,3}}}$$
\begin{figure}[h]
	\centerline{\includegraphics[width=0.8\linewidth]{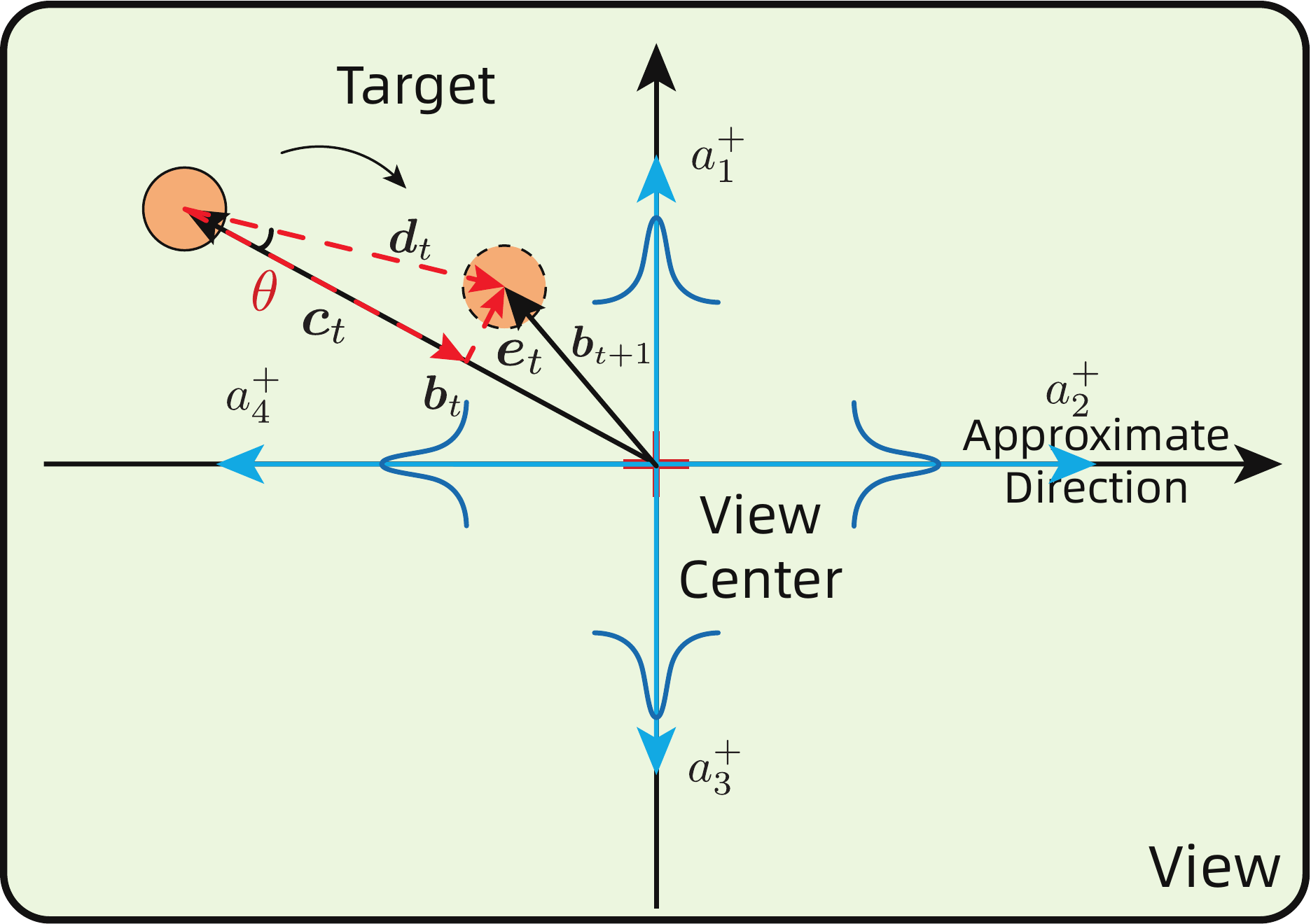}}
	\caption{Use the angle $\theta$ (red) between states of the target before and after action, the uncertainty of the basic controller can be estimated and then transform its action output into a Gaussian distribution (blue).}
	\label{scope2}
\end{figure}

After collecting the unit deviation of the whole initial exploration trajectories, we can get its uncertainty described by Gaussian parameters, mean $ \hat{\boldsymbol\mu} = \frac{1}{T}\sum_{t=1}^{T}\boldsymbol e_t/\boldsymbol a_t$ and covariance matrix $\hat{\boldsymbol{\Sigma}}=\text{diag}[\frac{1}{T-1}\sum^T_{t=1}(e_t/a_t - \mu)^2]$. Since these parameters stand for deviation of unit actions (unit actions = [1, 1, 1, 1]), the action distribution for the actual
actions is $\mathcal{G}_{bas}=(\boldsymbol\mu_{bas}, \boldsymbol\Sigma_{bas})= \mathcal{N}(a_{bas}(1+\hat{\boldsymbol{\mu}}), a_{bas}^2\hat{\boldsymbol{\Sigma}})$. Then the action outputs of basic controller are turned into distributions that easy to fuse (in IV.B).

\section{Experiment}
In this paper, we propose to train the continuum robot to aim at a target object by controlling the shift of multiple tendon drivers without the kinematic model, to track the movement of the target, and to maintain a certain axial distance. In minimally invasive surgery, the speciality of target tracking and axial distance keeping is critical for surgeons to concentrate on the practice since lesions will vibrate as the patient's breathing and other organ movements. To verify our idea, we first carry out experiments in a designed simulator and analyze ablations of it. Then we deploy it directly to a real-world continuum robot we designed. 

\subsection{Simulation}
We use \textit{Mujoco} to build a continuum robot model, with the physical manipulator to be referred. It can be divided into two motion sections, each of which is composed of 10 serial connected joints. Both of the sections are actuated by two sets of tendon-driven system at the end-point and have two-DOF separately. The panorama of the simulator is illustrated in the left of Fig. \ref{simulator}. We generate a set of space curves randomly by cubic spline (Fig. \ref{traj}). Points on the curves are all in the workspace in order to simulate the distance keeping task and ensure that the target is within reach.

\begin{figure}[ht]
	\centerline{\includegraphics[width=1\linewidth]{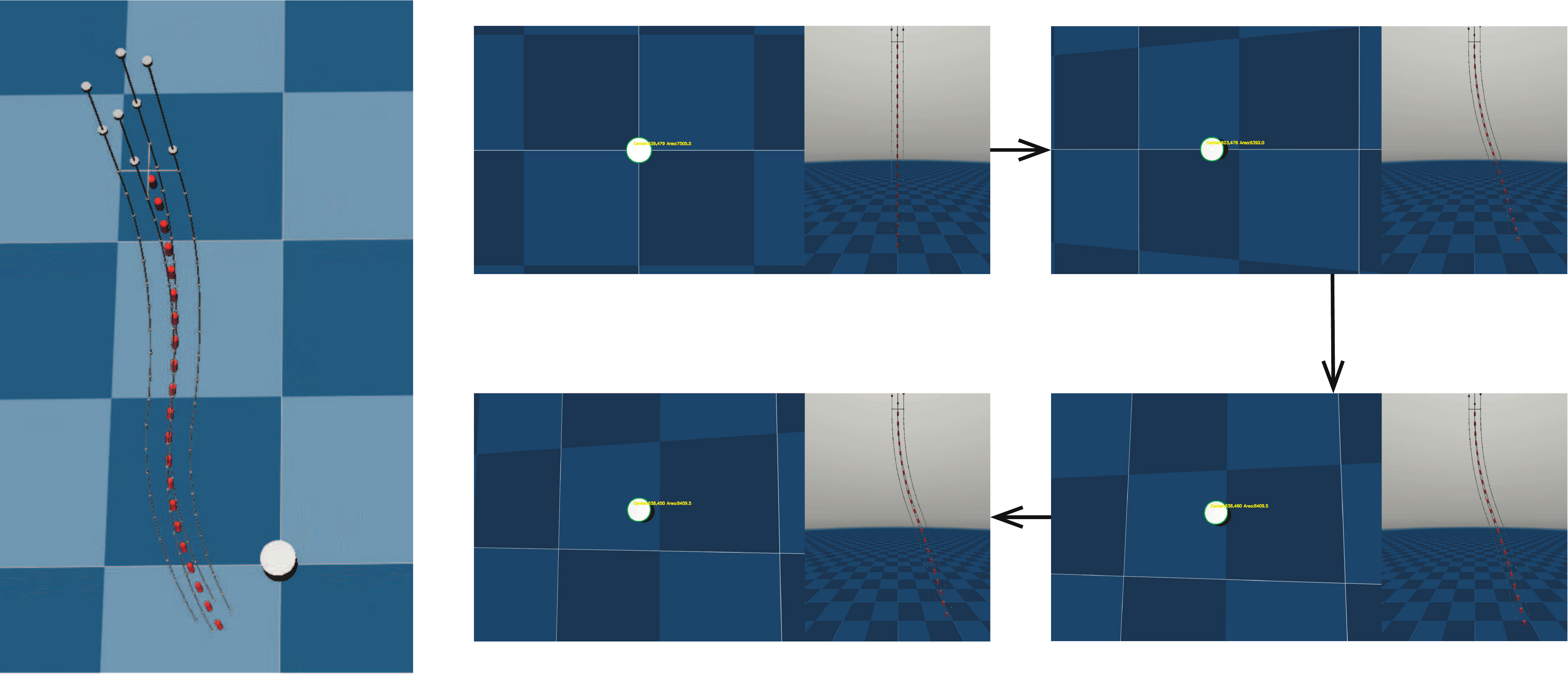}}
	\caption{Left: The panorama of tendon-driven continuum robot simulator based on Mujoco; Right: The continuum robot continuing tracking the target.}
	\label{simulator}
\end{figure}

\begin{figure}[ht]
	\centerline{\includegraphics[width=1\linewidth]{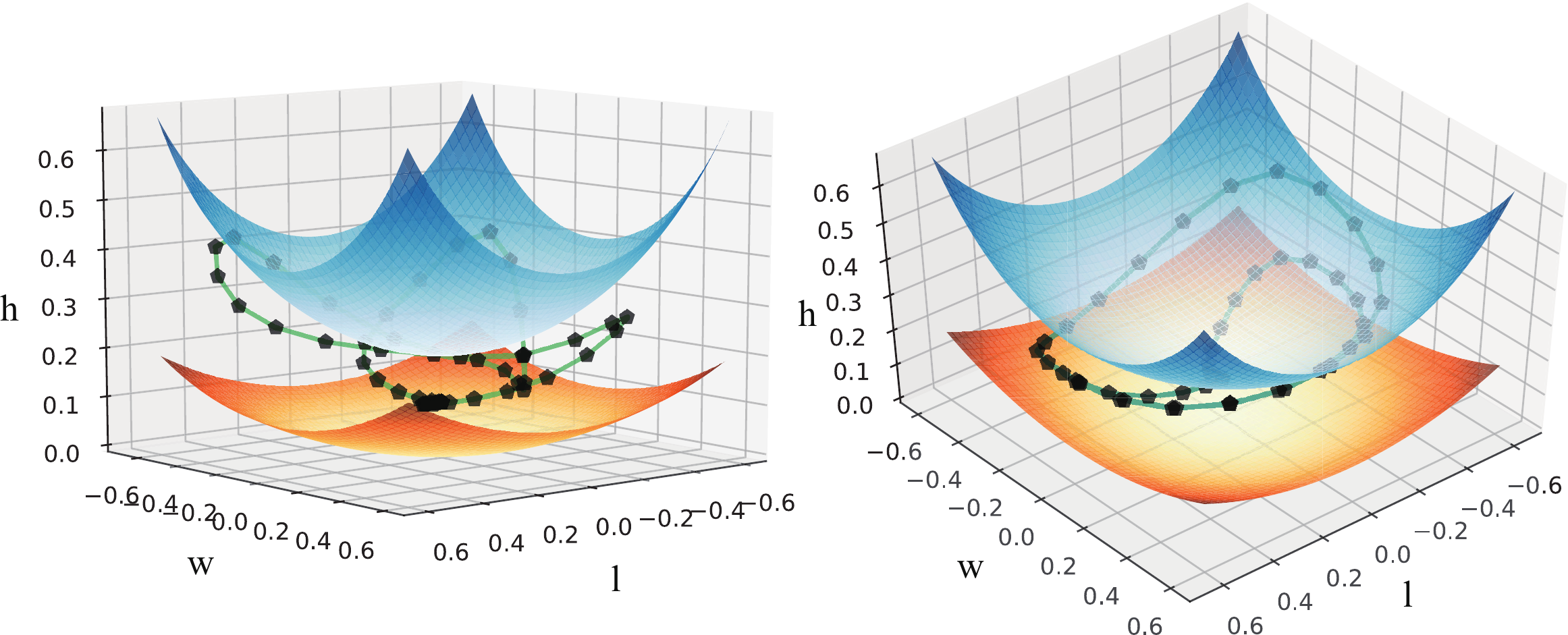}}
	\caption{Generate a set of cubic splines (green) randomly in the workspace (between two spheres), and let the target update its position follow these trajectories. The black hexagons are sequential positions of the target. Three-dimension trajectory in the workspace is propitious to estimate the distance keeping task.}
	\label{traj}
\end{figure}

Following existing studies, we use the epoch return to evaluate the performance of different algorithms. It calculates the transformation of 3-dimension Euclidean distance after each step, reward when the target reaches the visual center and punish when out of the field of view. To maintain the axial distance, the change of distance is also treated as a penalty. During the training process, each epoch has 1000 time steps with a 20 steps model rollout. Note that the performance comparison does not contain the exploration procedure. See the Appendix B for detailed parameters setting. 

Baselines we compared are as listed below:
\begin{itemize}
	\item \textbf{Basic Controller}: select actions followed the priors;
	\item \textbf{SAC}: the state-of-the-art MFRL algorithm using maximum entropy method;
	\item \textbf{PILCO}: a famous MBRL algorithm for robotics using GPs;
	\item \textbf{MBPO}: the state-of-the-art MBRL algorithm using ensemble BNN and adaptive length rollouts.
\end{itemize}

Therefore, the comparison of these five methods is shown in Fig. \ref{results}.

\begin{figure}[ht]
	\centerline{\includegraphics[width=1\linewidth]{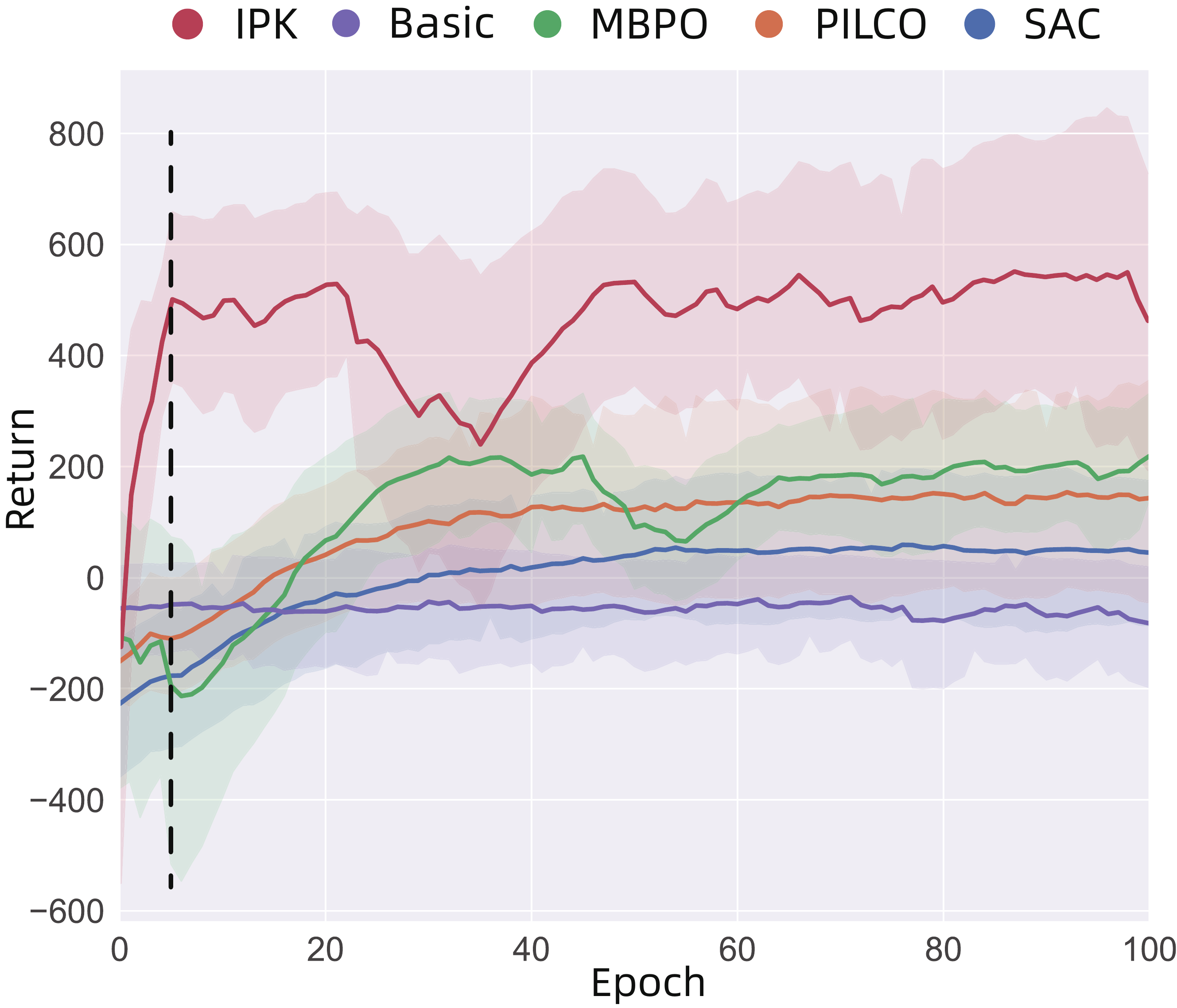}}
	\caption{The performance comparison among basic controller, SAC, PILCO, MBPO and our method. We smooth the curves by 0.9 weight parameter.}
	\label{results}
\end{figure}

From Fig. \ref{results}, we discover that by introducing inexplicit prior knowledge, the IPK improves faster than the others. At the second epoch, the IPK reaches and surpasses the basic controller more quickly than the other three SOTA algorithms. It also converges at a higher return after a few epochs. For the two MBRL algorithms, they both converge at a lower return. The primary reason is that both of them have not learned a comprehensive dynamic model because of the short task length. We will analyze this reason at Section IV.C in detail. Because of lacking the planning of the learned model, SAC's performance is worse than others as we assumed. 

Surprisingly, though acting followed rules, the performance of the basic controller is still at a low level. After checking videos saved during simulation, we find that it is very common for the continuum to oscillate near the origin under the control of the basic controller. Besides, the penalty of distance keeping also leads to this situation. Then a question occurs to us, how can these inexplicit prior knowledge assist our algorithm to gain such a high return?

From the video log, we figure out that it might benefit from the well-designed exploitation coefficient $\zeta$. Fusion controller can use the exploration of RL to amend the action value under soft constraints of rules, especially when it traps in vibration. In addition, since we merge the two motors of the same directions when calculating the KL-divergence, there is no constraint in selecting motors of the same direction. Therefore, it is easier for RL to explore a strategy to achieve the distance keeping task. The continuum can make $S$ type movements (Fig. \ref{simulator} Left) to shorten the distance. These are all the advantages of guaranteeing the exploration and we analyze this by visualizing the change of $\zeta$ below.

\subsection{Ablation Study \uppercase\expandafter{\romannumeral1}: Trade-off} 
The most critical part of the IPK framework is action fusion. We record the mean KL-divergence between IPK-MBPO (distinguished from the vanilla MBPO) and the basic controller, and the exploitation coefficient during the training process. 
\begin{figure}[h]
	\centerline{\includegraphics[width=1\linewidth]{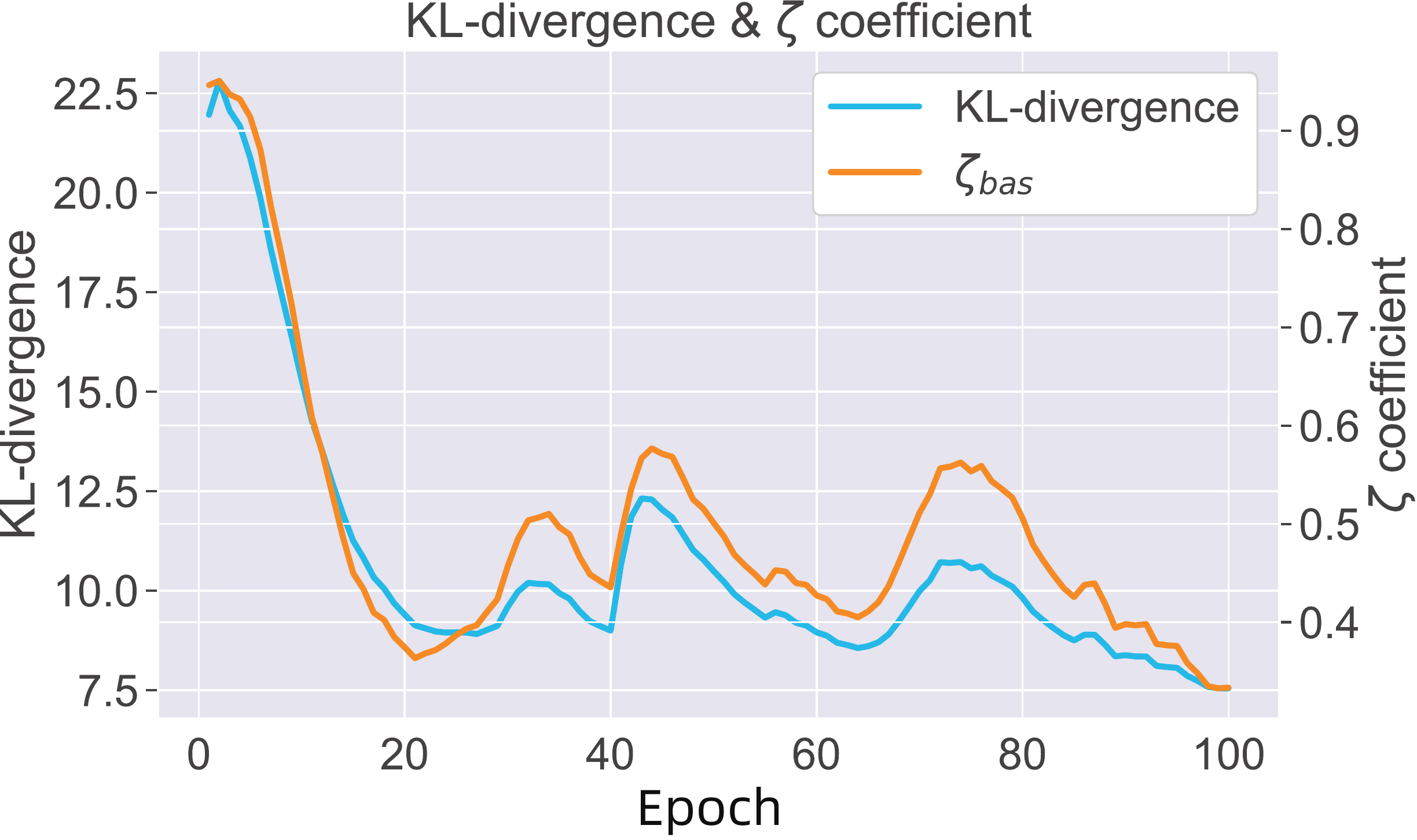}}
	\caption{The blue line represents the KL-divergence  between IPK-MBPO Gaussian policy and the basic policy. The orange line represents the value change of the exploitation coefficient $\zeta_{bas}$.}
	\label{KL_zeta}
\end{figure}

In Fig. \ref{KL_zeta}, we can discover that both KL-divergence and the exploitation coefficient descend through the training process. It demonstrates that the perfect performance of the fusion controller is not just relying on the basic controller but more focus on the data-driven IPK-MBPO  controller. Moreover, it also confirms the exploitation coefficient auto-adjustment theory in Theorem \ref{ecaa}.

\subsection{Ablation Study \uppercase\expandafter{\romannumeral2}: Task length}
In some cases, such as the target goes beyond the view scope, the environment needs to be reset. It means the current task is terminated. Short task length not only leads to low expected return but also impedes the model to learn the dynamics in the case with high curvature. Baseline algorithms are all wasting time on learning the dynamics near the plumb position again and again since the initial state of the task is vertically downward. As mentioned above, priors introducing intend to prolong the task length and to sample more successional action-state pairs. The average task length is recorded alongside the training process which is shown in Fig. \ref{task_length}.

\begin{figure}[h]
	\centerline{\includegraphics[width=1\linewidth]{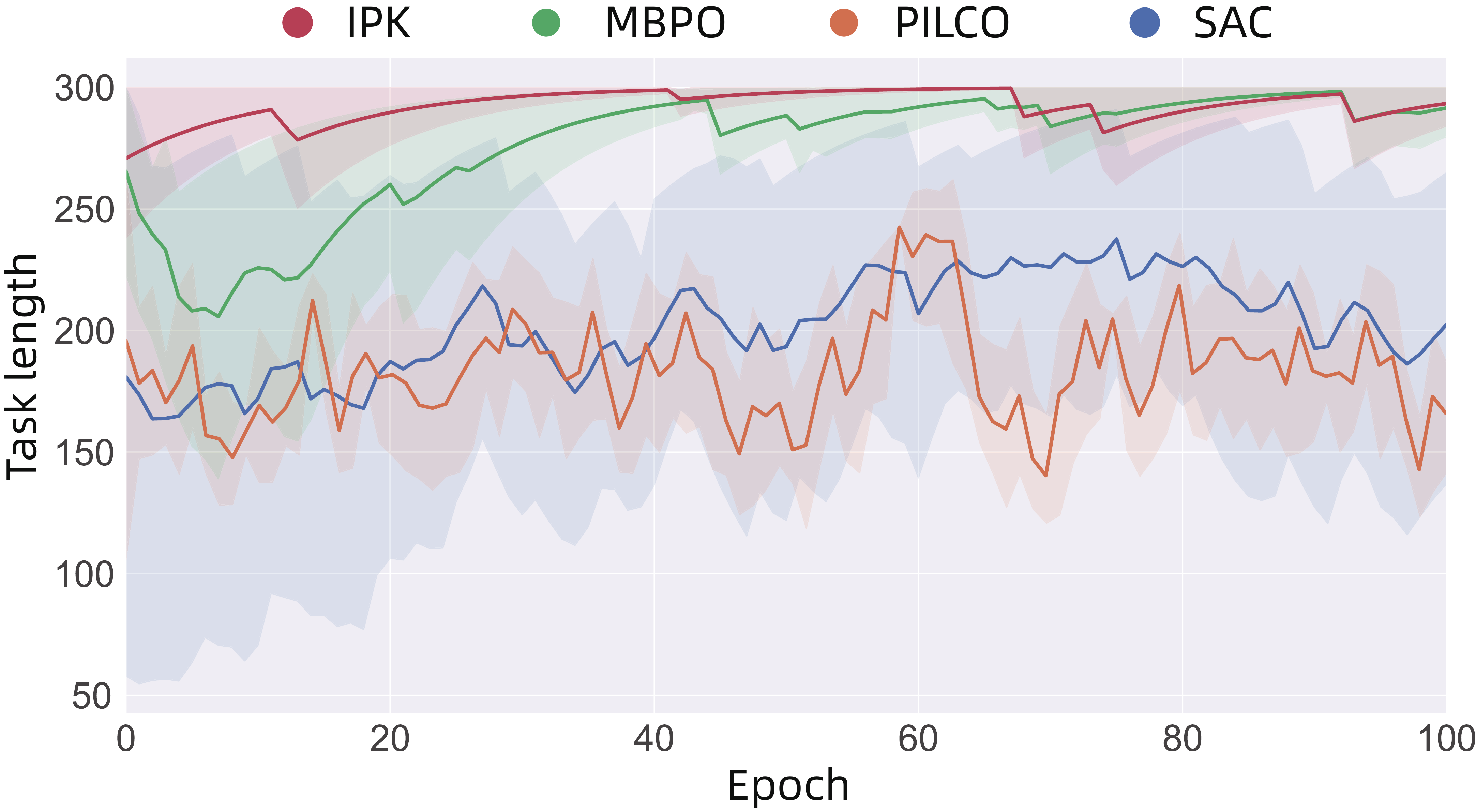}}
	\caption{Task length of each algorithm.}
	\label{task_length}
\end{figure}

The average task length of vanilla MBPO, PILCO and SAC are significantly shorter than our method. Due to the guidance of priors, the fusion controller can always reach the end of the task and seldom lose the target. It shows that exploration is safer when under the soft supervision or constraint of the priors.
 
\subsection{Experiment on real-world continuum robot}
\begin{figure*}[ht]
	\centerline{\includegraphics[width=1.0\linewidth]{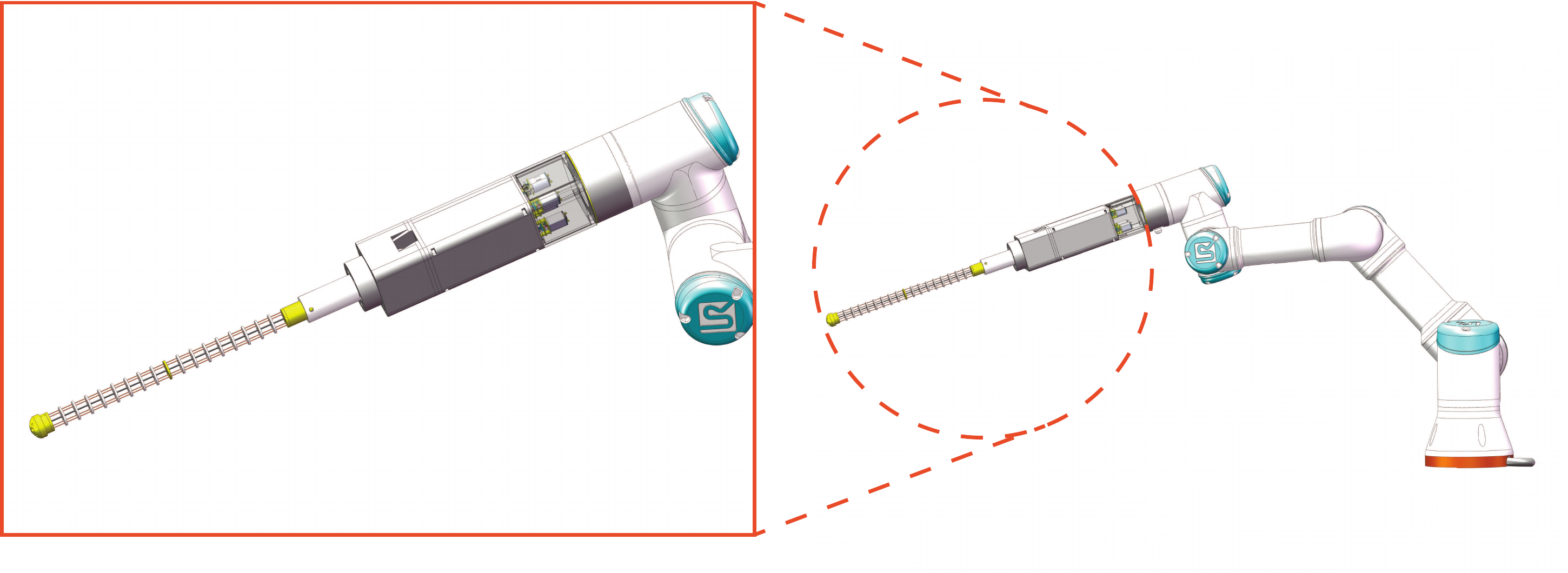}}
	\caption{The 3D model of the whole robot system. The continuum robot is connected with a UR (Universal Robots) commercial manipulator. There are four micromotors at the base of the continuum which offer 4-DOF by positive inversion.}
	\label{model}
\end{figure*}
To validate the effectiveness of the IPK algorithm, a real-world continuum robot is designed, the same as the mechanical structure described in simulation. The whole robot system is shown in Fig. \ref{model}. Plastic joints are evenly arranged and fixed on an elastic rod with large deflection, which provides necessary resilience as the backbone of the robot. Tendons are threaded through joints. Every two symmetrically arranged tendons attached to the same end-point can provide one DOF by producing strains in opposite directions. Transmission structures in such tendon-servo system sets are optimized by using screw rods with normal and reverse thread on both ends respectively. Then the two tendons linked with the same DOF can be driven by one servo motor, which avoids accuracy-loss caused by motor synchronization and structural redundancy. As a result, a one-to-one correspondence is formed between DOF and motors. The physical structure of this part is illustrated in Figure \ref{real_world_env}(a). In this way, the structure of the continuum manipulator is greatly simplified with lighter weight and higher accuracy to fit the simulation within the error range.

Same as simulation, the continuum robot needs some extra devices to perceive the experiment environment. A pinhole camera is fixed on the end-point to gather information for tracking tasks. Encoders on servo motors are used to ensure the IPK actions to be executed precisely, and protect the manipulator from being damaged in over range conditions. An extra camera is set up towards the robot, only for result evaluation. The gathered information is shown in Fig. \ref{real_world_env}(b).

\begin{figure}[ht]
	\centerline{\includegraphics[width=1.0\linewidth]{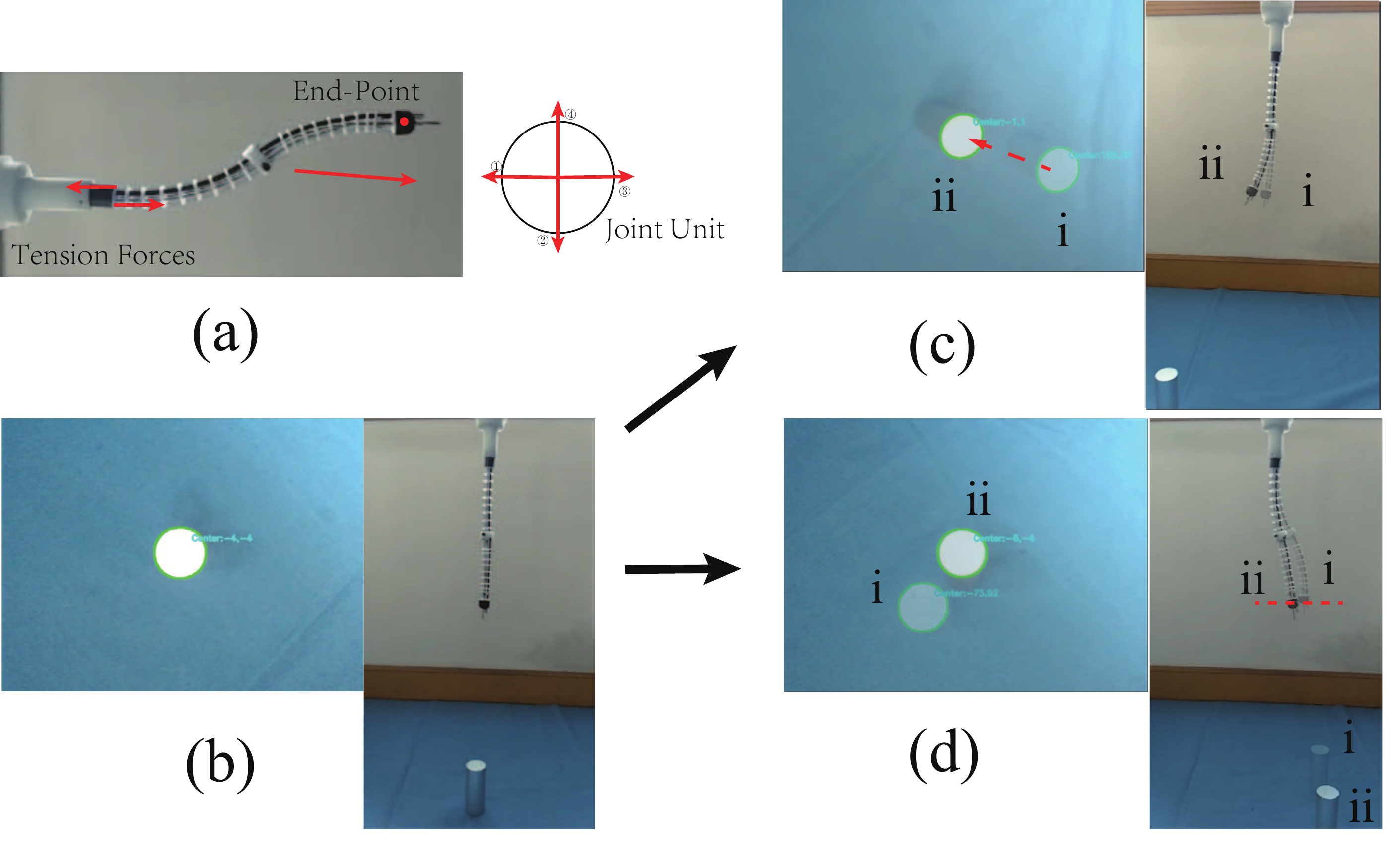}}
	\caption{Devices and final results in real-world experiments. (a) Structure of the continuum robot. (b) Training processes in every epoch started from the zero position. (c) Mode I: The continuum manipulator was trained to tracking the target by using the visual observation, moving from state i to state ii for example. (d) Mode II: Height information of the end-point was gathered by image processing. With height information added into rewards, the continuum manipulator was trained to keep the distance during tracking tasks. Notice that state ii has a similar height with state i. Details as video.}
	\label{real_world_env}
\end{figure}
The real-world experiment process was similar to that in simulation. Different from sim-to-real studies, the real-world model is not transformed from the simulation but directly learn from the real environment. In this case, the model can learn the uncertainty in the real environment and take these errors into account. The experiment was carried out mainly in two steps. Firstly, the continuum manipulator was trained in tracking objects by using visual observation (Fig. \ref{real_world_env}(c)). In this case, in order to shorten the training process, the real object was replaced by a screen that kept playing the video of simulated objects in a loop. Once the tracking task was failed or mechanical limits were reached, the manipulator would come back to the zero position with the help of encoders and prepare for the next training epoch. With the prior-experience of basic actions, after only half an hour, 10,000 steps, the robot gained an acceptable performance.
Secondly, based on the already learned model in tracking tasks, height information of the end-point was added into rewards and made the manipulator learned to keep axial distance with the object (Fig. \ref{real_world_env}(d)). Then the robot would try to track the object with the least distance loss. After one half and an hour, 20,000 steps, the robot achieved convergence. 
Finally, the network weights were saved to reproduce the two tasks. The video of simulation and real-world experiments is available at \textit{\url{https://youtu.be/MhqBSI-SXQc}}.

\section{Conclusion}
In this paper, we propose a new approach to integrating human prior knowledge into reinforcement learning robotics control. It takes full advantage of inexplicit prior knowledge and accelerates the learning process by guiding the policy search towards the approximate correct directions. Furthermore, the exploration of MBRL is also ensured by some learned coefficients. An empirical result which is given by visualizing the KL-divergence between action distributions proves our theory. By achieving the experiment we conducted, the designed continuum robot can assist the minimally invasive surgery. 

Despite the delicate framework designing, the success is still merely proved in simple action space. We will focus on how to learn a combinatorial control of UR and continuum assembly efficiently and achieve the whole minimally invasive surgical procedure in our future work.

\section*{Acknowledgment}
	This work was supported by the National Natural Science Foundation of China
	(Grant No. 61973210), Shanghai Science and Technology Commission (Grant
	No. 17441901000), the Medical-engineering Cross Projects of SJTU (Grant
	Nos. YG2019ZDA17, ZH2018QNB23), and the Scientific Research Project
	of Huangpu District of Shanghai (Grant No. HKQ201810).
	\bibliographystyle{ieeetr}
	\bibliography{ref_IPK1}
	
\newpage
\onecolumn
\appendix
\subsection{Proofs}
\subsubsection{Theorem 3.1}
	Let $\mathcal{G}_{gau}$ be the Gaussian distribution from the $T-1$ Gaussian policy and let $\mathcal{G}_{bas}$ be the basic action distribution with uncertainty. Then the exploitation coefficient is related to the KL-divergence between these two distributions.
\begin{equation}
\begin{aligned}
\zeta_{bas}^{T^*}&=\arg\min_{\zeta^{T}\geq  0} \mathbb{E}_{s_{T-1},a_{T-1}\sim\rho_\pi^*}
\{-\zeta^{T}_{bas}\mathcal{D}_{KL}[\mathcal{G}_{bas}(\pi^{T-1})||\mathcal{G}_{gau}(\pi^{T-1})]-\zeta^{T}_{bas}\mathcal{D}_0\}\\
\zeta_{real}^{T^*}&=1-\zeta_{bas}^{T^*}
\end{aligned}
\end{equation}
where $\mathcal{D}_0$ is a hyperparameter for KL-divergence limiting.\\
\\
\begin{proof}
    We aim to find a maximum entropy policy with a maximal expected return that satisfies a minimum distance between RL policy and inexplicit prior knowledge. This can be formalized as a constrained optimization problem.
	\begin{equation}
		\begin{aligned}
		&\max_{\pi_{0:T}}\mathbb{E}_{\rho_\pi}[\sum_{t=0}^Tr_{\mathcal{H}}(s_t,a_t)]\\
		&s.t.\quad \mathbb{E}_{(s_t,a_t)\sim\rho_\pi}[-\mathcal{D}_{KL}(\pi_{bas}||\pi_{gau})]\leq\mathcal{D}_0
		\end{aligned}
	\end{equation}
	where $r_{\mathcal{H}}$ is the reward with maximum entropy which is defined in SAC as $r_{\mathcal{H}}\left(\mathbf{s}_{t}, \mathbf{a}_{t}\right) \triangleq r\left(\mathbf{s}_{t}, \mathbf{a}_{t}\right)+\mathbb{E}_{\mathbf{s}_{t+1} \sim p}\left[\mathcal{H}\left(\pi\left(\cdot | \mathbf{s}_{t+1}\right)\right)\right]$.
	
 Use Lagrange Multiplier Method to transform the constrained optimization problem into the unconstrained optimization problem.
	\begin{equation}
		\begin{aligned}
		\max_{\pi_T}&\mathbb{E}_{(s_t,a_t)\sim\rho_{\pi}}[r(s_t,a_t|\zeta_T)]=\\&\min_{\zeta\geq0}\min_{\alpha_T\geq0}\max_{\pi_T}\mathbb{E}[r(s_t,a_t|\zeta_T)-\alpha_T\log\pi_{gau}(a_t|s_t)-\zeta_{T}\mathcal{D}_{KL}(\pi_{bas}||\pi_{gau})]-\alpha_T\bar{\mathcal{H}}-\zeta_{T}\mathcal{D}_0
		\end{aligned}
	\end{equation}
	
Therefore,  the coefficient $\zeta$ $(\zeta = \zeta_{bas})$ can be optimized by proceeding recursion after obtaining the optimal policy and Q function. Note that the $T$ step optimal $\zeta$ is related to the $T-1$ step optimal policy.
	\begin{equation}
		\zeta^*_T=\arg\min_{\zeta_T}\mathbb{E}_{a_t\sim\pi^*_{T-1}}[-\zeta_T\mathcal{D}_{KL}(\pi_{bas}^{T-1}||\pi_{gau}^{T-1})-\zeta_T\mathcal{D}_0]
	\end{equation}
\end{proof}

\subsubsection{Theorem 3.2}
	According to Theorem \ref{ecaa} and Equation \ref{weighted_kalman}, let $\pi_{T}\triangleq \pi(a_t|s_t,\zeta_{bas}^{T-1})$, the new policy of $T+1$ time step is $\pi_{T+1}\triangleq \pi(a_{T+1}|s_{T+1},\zeta_{bas}^{T})$. Then $Q^{\pi_{T+1}}\left(\mathbf{s}_{t}, \mathbf{a}_{t}\right) \geq Q^{\pi_{T}}\left(\mathbf{s}_{t}, \mathbf{a}_{t}\right)$ for all $\left(\mathbf{s}_{t}, \mathbf{a}_{t}\right) \in \mathcal{S} \times \mathcal{A}$ with $|\mathcal{A}|<\infty$.\\
	\\
\begin{proof}
 Similar to the soft Bellman equation, we expand the Q value function to show the relationship with the exploitation coefficient $\zeta$. Here $\zeta$ is short for $\zeta_{bas}$.
	\begin{equation}
		\begin{aligned}
		Q^{\pi_T}&=r(\zeta^{T})+\gamma \mathbb{E}_{s_t\sim\rho_{\pi_T}}[V^{\pi_T}(s_{t+1},a_{t+1})]\leq r(\zeta^{T+1})+\gamma \mathbb{E}_{s_t\sim\rho_{\pi_T}}[ \mathbb{E}_{a_t\sim\rho_{\pi_{T+1}}}[Q^{\pi_T}(s_t,a_t)-\mathcal{H}^T-\mathcal{D}_{KL}(\zeta^{T+1})]]\\
		&\leq Q^{\pi_{T+1}}
		\end{aligned}
	\end{equation}
	
	Obviously, along with the improvement of the fusion policy, the KL-divergence between original policy and fusion policy is getting lower, leading to a smaller $\zeta$. Therefore, a higher Q value is guaranteed.
\end{proof}
\newpage
\subsection{Hyperparameter}
\begin{table}[htbp]
	\begin{center}
		\caption{Hyperparameter}
	\begin{tabular}{ccc}
		\toprule[1pt] Catalog &Parameter & Value\\
		\midrule[1pt] \multirow{6}*{Network} &optimizer & Adam\\
		~&learning rate & $3 \times 10^{-4}$ \\
		~&number of hidden layers (all networks) & 2 \\
		~&number of hidden units per layer & 256 \\
		~&number of samples per minihatch & 256 \\
		~&nonlinearity & Sigmoid \\
		\midrule[1pt] \multirow{9}*{Training}&discount $(\gamma)$ & 0.99 \\
		~&target entropy  & -2 \\
		~&target divergence  & -1.5 \\
		~&task length & 300\\
		~&epoch length& 1000\\
		~&rollout length & 20 \\
		~&frequency of model training & 250\\
		~&number of BNN models & 7\\
		~&replay buffer size & $10^{6}$ \\
		~&initial exploration steps & 600\\
		\bottomrule[1pt]
	\end{tabular}
\end{center}
	\label{hyper}
\end{table}

\subsection{Pseudocode}
	\begin{algorithm}[h]
	\caption{Efficient Model-based Reinforcement Learning based on Inexplicit Prior Knowledge}
	\label{Pseudocode}
	\begin{algorithmic}[1]
		\State Initialize environment data pool $\mathcal{D}_{env}$
		\State Initialize model data pool $\mathcal{D}_{model}$
		\State Initialize fake env by ensemble BNN models
		\State Initialize exploitation coefficient $\zeta$
		
		\While{poolsize $<$ init exploration steps} \algorithmiccomment{Initial exploration procedure}
		
		\parState{%
			Select actions $ a_{bas}\sim \pi_{bas}(\cdot|o_{real}) $}
		\parState{%
			$ o_{real}',  r_{real} = \textbf{real env}(o_{real},a_{bas})$}
		\parState{%
			Estimate $ o_{mbpo}',  r_{mbpo}$}
		\parState {Store ($o_{real}, a_{uni}, r_{mbpo}, o_{mbpo}', a_{bas} ,r_{real},o_{real}'$) in $\mathcal{D}_{env}$}
		
		\EndWhile
		
		\parState{Estimate the accuracy of $\pi_{bas}$ by $\mathcal{D}_{env}$ and turn into means $\boldsymbol{\mu}$ and variances $\boldsymbol{\sigma^2}$\algorithmiccomment{Uncertainty estimation}}

		\For{epochs $ =1,  M $} \algorithmiccomment{Learning \& Fusion  procedure}
		\State Reset the environment, and get initial $ o $

		\For{timesteps $=1,  N $} \algorithmiccomment{Training BNN models}
		
		\If{timesteps $\%$ model train freq $==0$ }
		\State Train BNN models on $\mathcal{D}_{env}$ with real data
		\State Perform k-step model rollout; add to $\mathcal{D}_{model}$
		\EndIf
		
		\parState {%
			$ a_{bas}\sim \pi_{bas}(\cdot|o) $; $ \mathcal{G}_{bas}= \mathcal{N}(a_{bas}(1+\boldsymbol{\mu}), a_{bas}^2\boldsymbol{\Sigma}) $\\
			Update coefficient $\zeta\varpropto KL(\mathcal{G}_{bas}||\mathcal{G}_{gau})$\\
			$ \mathcal{G}_{fus}=Kalman(\mathcal{G}_{gau}, \mathcal{G}_{bas}| \zeta)$; \\Sample action $a_{fus}, a_{gau}$ from $\mathcal{G}_{fus}, \mathcal{G}_{gau}$
		}
		\parState{%
		$ o_{real}',  r_{real} = \textbf{real env}(o_{real},a_{fus})$}
		\parState{%
			Estimate $ o_{mbpo}',  r_{mbpo}$}
	
		\parState{Store ($o_{real}, a_{gau}, r_{mbpo}, o_{mbpo}', a_{fus} ,r_{real},o_{real}'$) in $\mathcal{D}_{env}$}
	
		\If{ready to train} \algorithmiccomment{Training SAC}
		\parState{%
			Train soft Actor-Critic on $\mathcal{D}_{model}$ with concatenate data of real and mbpo}
		\EndIf
		
		\EndFor
		
		\EndFor


	\end{algorithmic}
\end{algorithm}

\end{document}